\documentclass[11pt]{article}
\usepackage[utf8]{inputenc}
\usepackage[T1]{fontenc}
\usepackage{lmodern}
\usepackage[margin=1in]{geometry}
\usepackage[authoryear,round]{natbib}
\usepackage{amsmath,amssymb,amsfonts,amsthm,mathtools}
\usepackage{bbm}
\usepackage{booktabs}
\usepackage{array}
\usepackage{multirow}
\usepackage{graphicx}
\usepackage{subcaption}
\usepackage{wrapfig}
\usepackage{float}
\usepackage{adjustbox}
\usepackage{enumitem}
\usepackage[normalem]{ulem}
\usepackage{nicefrac}
\usepackage{microtype}
\usepackage{xcolor}
\usepackage{url}
\usepackage{tikz}
\usetikzlibrary{arrows}
\usepackage{algorithm}
\usepackage{algpseudocode}
\usepackage[most]{tcolorbox}
\usepackage{hyperref}
\usepackage[capitalize,noabbrev]{cleveref}

\newcolumntype{P}[1]{>{\centering\arraybackslash}p{#1}}
\newcolumntype{M}[1]{>{\centering\arraybackslash}m{#1}}

\def\ddefloop#1{\ifx\ddefloop#1\else\ddef{#1}\expandafter\ddefloop\fi}
\def\ddef#1{\expandafter\def\csname bb#1\endcsname{\ensuremath{\mathbb{#1}}}}
\ddefloop ABCDEFGHIJKLMNOPQRSTUVWXYZ\ddefloop
\def\ddef#1{\expandafter\def\csname c#1\endcsname{\ensuremath{\mathcal{#1}}}}
\ddefloop ABCDEFGHIJKLMNOPQRSTUVWXYZ\ddefloop

\def\E{\mathbb{E}}

\def\eps{\epsilon}

\definecolor{darkblue}{rgb}{0.0,0.0,0.65}
\definecolor{darkred}{rgb}{0.68,0.05,0.0}
\definecolor{darkgreen}{rgb}{0.0,0.29,0.29}
\definecolor{darkpurple}{rgb}{0.47,0.09,0.29}
\hypersetup{
   colorlinks = true,
   citecolor  = darkblue,
   linkcolor  = darkred,
   filecolor  = darkblue,
   urlcolor   = darkblue,
}

\theoremstyle{plain}
\newtheorem{theorem}{Theorem}[section]

\newtheorem{lemma}[theorem]{Lemma}

\theoremstyle{definition}
\newtheorem{definition}[theorem]{Definition}

\theoremstyle{remark}
\newtheorem{remark}[theorem]{Remark}

\Crefname{algorithm}{Algorithm}{Algorithms}

\title{Towards Distillation Guarantees under Algorithmic Alignment for Combinatorial Optimization}

\author{
Thien Le\\
SEAS, Harvard University\\
Cambridge, MA\\
\texttt{thien\_le@seas.harvard.edu}
\and
Melanie Weber\\
SEAS, Harvard University\\
Cambridge, MA\\
\texttt{mweber@seas.harvard.edu}
}

\date{}

\begin{document}

\maketitle

\begin{abstract}
Distillation transfers knowledge from a large model trained on broad data to a smaller, more efficient model suitable for deployment. In structured prediction settings, prior knowledge about the task can guide the choice of a target architecture that is algorithmically aligned with the underlying problem. Building on recent learning-theoretic analyses of decision-tree (DT) distillation \citep{boixadsera2024theorymodeldistillation}, we study when distillation succeeds for combinatorial optimization tasks. We focus on the case where the target model is a graph neural network whose architecture is aligned with a dynamic programming (DP) algorithm for the task. Assuming that the source model is sufficiently rich, formalized through the linear representation hypothesis (LRH) \citep{elhage2022toy,park2024lrh}, we show that the distillation problem can be solved efficiently in the complexity parameters of the DP transition function, represented as a DT. Our results provide a rigorous sufficient condition for successful distillation in the flavour of algorithmic alignment.  
\end{abstract}

\section{Introduction}

Modern machine learning often benefits from incorporating useful structure into the learner, as demonstrated by a wide range of models, from convolutional architectures in vision to graph neural networks (GNNs) for relational data \citep{krizhevsky2012imagenet,lecun2015deep,goodfellow2016deep,gilmer2017gnn,kipf2017semisupervised}. At the same time, the success of general-purpose models has made it natural to ask whether such structure can instead be learned from data and subsequently transferred to smaller models \citep{bachmann2023scaling,brehmer2025does,hinton2015distilling,jiao2020tinybert}. The ``learn first, distill later'' paradigm offers a middle ground: train a large source model, then distill its knowledge into a target model whose architecture embodies the desired inductive bias.

In this work, we ask when the knowledge learned by a large neural network on graph algorithmic tasks can be distilled into a smaller graph model whose architecture is aligned with the underlying algorithm. We answer this question affirmatively for a class of local-iteration graph algorithms, a dynamic-programming abstraction that encompasses simple reachability-type computations. In particular, we show that if the learned source representation linearly exposes the elementary components of the algorithm, then PAC-distillation \citep{boixadsera2024theorymodeldistillation} into an algorithmically aligned GNN is tractable in the complexity parameters of the local transition rule, under the restricted parameter regime made explicit below. This provides a rigorous sense in which the target architecture not only compresses the source model, but also leverages the algorithmic structure of the task.

\subsection{Contribution of this paper}

We focus on the setting in which the ground truth is computed by a \textit{local-iteration algorithm} (\Cref{alg:local_iteration}). This is a dynamic programming (DP) algorithm whose states are indexed by pairs $(t,v) \in [l]\times V$, where $t$ denotes an iteration of message-passing and $v$ is a vertex of the input graph $G=(V=[n],E)$. We assume that the DP transition function is represented by a small decision tree of depth $r$, for instance, in the DP that solves graph reachability.

\begin{algorithm}[t]
\caption{Local-iteration algorithm $\mathcal{A}^l[g]$}
\label{alg:local_iteration}
\textbf{Input} Initialization vector $\textsc{Init}$, graph $G=([n],E)$\\
\textbf{Output} $\{0,1\}$ classification
\begin{algorithmic}
\ForAll{$v \in [n]$}
  \State $h_{v,0} \gets \textsc{Init}(v)$
\EndFor
\For{$t = 1$ to $l$}
  \ForAll{$v \in [n]$}
    \State $h_{v,t} \gets g((h_{u,t - 1})_{u \in \mathcal{N}(v)}, G, v)$  \Comment{$h_{v,t}$ denotes state of $v$ after iteration $t$}
  \EndFor
\EndFor
\State \textbf{return} {$h_{n,l}$}
\end{algorithmic}
\end{algorithm}


The central question that we address in this paper is as follows:

\begin{tcolorbox}
Suppose the local transition rule of a local-iteration graph algorithm is given by a decision tree $g$. Can a large model trained on tasks generated by this algorithm be tractably distilled into a smaller graph model, with complexity controlled by the complexity of $g$?
\end{tcolorbox}

We make the following contributions:
\begin{enumerate}
    \item We give a sufficient condition, based on a LRH for the source network, that enables distillation of certain local algorithms. This condition, which we call \emph{local-iteration alignment} (\Cref{def:local_iteration_alignment}), requires the source representation to linearly represent simple functions of the form $\mathcal{A}^l[b]$, where $b$ is a single conjunction (\Cref{fig:local_iteration}, right).
    \item We prove that under this condition, distillation from the large source network to a GNN is tractable in a restricted structural regime (\Cref{thm:efficient_distillation}). In contrast, efficiency is not guaranteed under algorithmic mismatch, such as when learning the same local-iteration computation with ordinary decision trees (\Cref{lem:gnn_vs_dt}).
    \item We give an explicit $(\epsilon,\delta)$-distillation algorithm for the framework (\Cref{alg:distillation}). It draws $\text{poly}(1/\epsilon, \log(1/\delta), n, l, r)$ samples and runs in time $\text{poly}(s^n, 2^{nlr},1/\epsilon, \log(1/\delta))$ where $n$ is the size of the graph, $s$ and $r$ are the size and depth of the true inner tree and $l$ is the number of rounds of message-passing. 
\end{enumerate}

We complement our theoretical results with experiments that primarily examine whether the proposed LRH assumption emerges in a learned source model, together with a small end-to-end implementation of the proposed distillation algorithm.



\begin{figure}
    \centering
    \includegraphics[width=0.7\linewidth]{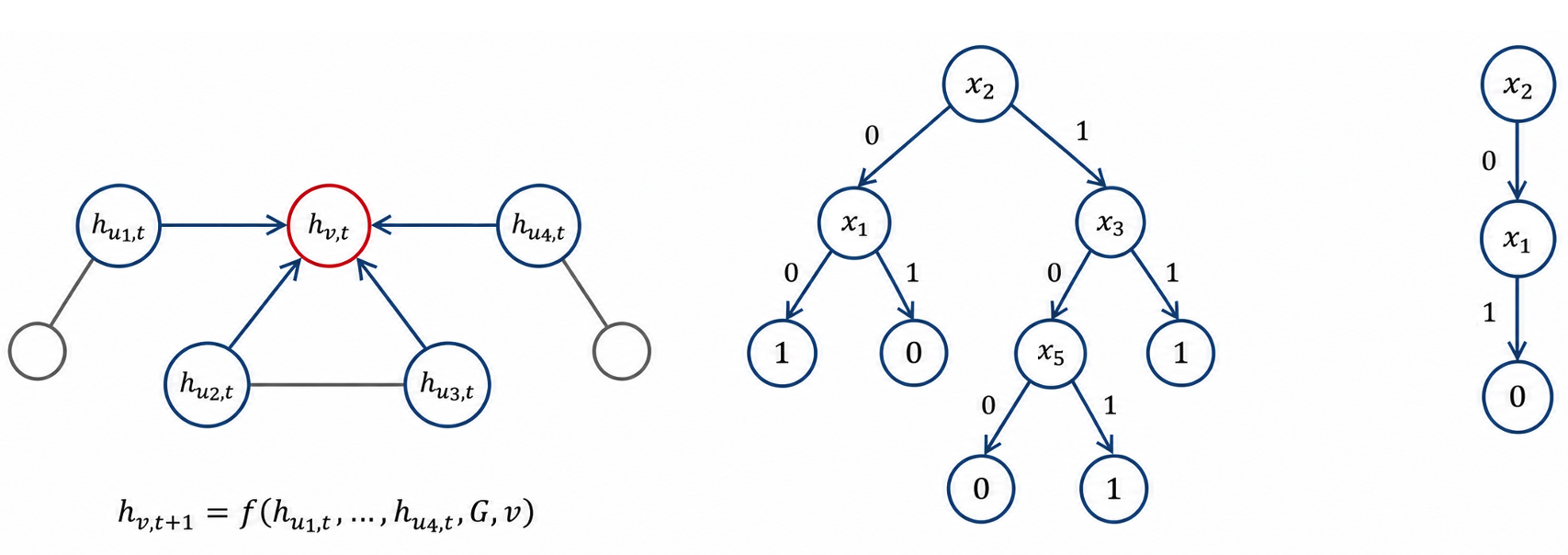}
    \caption{
    \textbf{Illustration of \Cref{alg:local_iteration}.} (Left) A GNN that aggregates neighboring information at each iteration. Its architecture reflects the inductive bias of our local-iteration algorithms. (Middle) In our setting, the aggregation function $g$ is represented by a decision tree. (Right) A root-prefix path in the decision tree. The path in this example can be viewed as a degenerate decision tree that outputs $0$ (the leaf constant) iff $\neg x_2 \wedge x_1$ ($x_2$ is negated because it takes the edge with label $0$, $x_1$ takes the edge labeled $1$ and thus not negated). Our hypothesis requires the source model to be rich enough to linearly represent simple local-iteration algorithms whose aggregation functions correspond to such root-prefixed paths 
    }
    \label{fig:local_iteration}
\end{figure}
\subsection{Related works (See Extended Background \Cref{app:extended_background} for more details)}

\textbf{Distillation.} Model distillation asks whether the behavior of a large trained model can be transferred to a smaller target model \citep{hinton2015distilling,jiao2020tinybert}. We adopt the PAC-distillation framework of \citet{boixadsera2024theorymodeldistillation}, which assumes that the source model exposes useful latent features through a LRH~\citep{elhage2022toy,park2024lrh}. This hypothesis is motivated by the empirical observation that neural representations often encode meaningful concepts as linear directions, as seen for example in word embeddings and debiasing applications \citep{mikolov2013word2vec,tolga2016man,manzini2019black,liang2020debiasing,chuang2023debiasing}. In our setting, the relevant linear features are not arbitrary concepts: they are simple local-iteration subroutines.

\textbf{Neural Combinatorial Optimization.} 
Many graph optimization algorithms, such as Bellman-Ford, iteratively propagate local information across the graph, much like the layers of a message-passing GNN. An aligned model can therefore reuse this algorithmic scaffold and focus on learning the local update rule \citep{xu2020algorithmicalignment}. This perspective has motivated theoretical work on algorithmic alignment \citep{valiant1984pac,dudzik2022gnnsdp,dudzik2024cocycles} as well as practical neural heuristics \citep{kahng2024steiner,nerem2025graphneuralnetworksextrapolate,he2025primaldual,gasse2019co}. We study the same principle in a distillation setting: rather than asking whether a GNN can learn an algorithm from scratch, we ask when a trained source model can be efficiently distilled into a GNN.

\section{Preliminaries}

In general, we denote by $\mathcal{X}$ the input set and by $\mathcal{Y}$ the label set. In this paper, we focus on binary classification, i.e., $\mathcal{Y} = \{0,1\}$. We will write $[n] := \{1, 2, \ldots, n\}$. We define $\mathcal{G} = \mathcal{G}_n$ as the space of all simple labeled graphs on $n$ vertices. 
For some dimension $m$, we say that a neural network defines an (abstract) latent representation $\varphi: \mathcal{X} \to \mathbb{R}^m$. 

Following standard PAC-learning notation, we will usually denote by $\mathcal{C}$ the concept class (class of possible ground truths), $\mathcal{H}$ the hypothesis class (output range of the learning algorithm), and $\mathcal{D}_c \in \mathcal{P}({\mathcal{X}} \times \mathcal{Y})$ the input distribution for some ground truth $c \in \mathcal{C}$. When there is a vector of many inputs elements $S \in \mathcal{X}^N$ we apply $c$ pointwise and write $c(S) := (c(S_i)_{i \in [N]})$. In the setting of distillation, we have a source class $\mathcal{F}$ and a target class $\mathcal{H}$. 

\paragraph{PAC learning} In the traditional PAC learning framework~\citep{valiant1984pac}, the \emph{concept class} is $(\epsilon, \delta)$-learnable with $n$ samples if there is an algorithm $\mathcal{A}$ such that for any distribution $\mathcal{D}$ over the input and any concept in $c \in \mathcal{C}$, 
\begin{equation}
    \Pr_{S \sim \mathcal{D}^n} [\textup{error}_{c,\mathcal{D}} (\mathcal{A}(S, c(S)) \leq \epsilon] \geq 1- \delta.
\end{equation}
Here, the error function is the $0$-$1$ population risk: $\textup{error}_{c, \mathcal{D}}(f) := \Pr_\mathcal{D}[f(x) \neq c(x)]$. 

\paragraph{PAC-distillation} PAC-distillation~\citep{boixadsera2024theorymodeldistillation} is a relaxation of PAC learning in which one assumes access to a successful source model class $\mathcal{F}$ to train a target class $\mathcal{H}$ by finding an algorithm $\mathcal{A}$ such that for any distribution $\mathcal{D}$ on $\mathcal{X}$, any source $f \in \mathcal{F}$, 
\begin{equation}
    \Pr_{S \sim \mathcal{D}^n} [\textup{error}_{f,\mathcal{D}} (\mathcal{A}(S, f)) \leq \epsilon] \geq 1- \delta.
\end{equation}
Such an algorithm is said to $(\epsilon, \delta)$-distill $\mathcal{F} \to \mathcal{H}$. Note that since the algorithm has access to the successful model $f$, giving a PAC distillation algorithm is easier than giving a PAC algorithm since one can just use $f$ to query labels and simulate PAC learning. The advantage of this framework is to sidestep some of the hardness results of PAC learning by leveraging 
structures in the class $\mathcal{F}$, such as the LRH that is widely observed in practice \citep{elhage2022toy,tolga2016man}. Concretely, \citet{boixadsera2024theorymodeldistillation} was able to obtain a PAC-distillation algorithm for decision trees in time $\mathrm{poly}(d, 2^r)$ (\Cref{thm:dt_distillation}) even though it is open whether they are PAC-learnable in less than $d^{O(r)}$ time \citep{weisberg2020distributionfree}

In practice, $\mathcal{F}$ can be thought of as large, pre-trained neural networks that have achieved low errors on some tasks, and the target class $\mathcal{H}$ can be understood as a function class with inductive bias that can more efficiently represent the ground truth, for example, invariant neural networks such as CNNs or GNNs. Distillation then asks if there are efficient algorithms to find a good representation of the ground truth in the target class.

\paragraph{Algorithmic alignment and LRH} In this paper, we view algorithmic alignment through the lens of the LRH, formally defined as follows: 
\begin{definition}[$\tau$-LRH~\citep{boixadsera2024theorymodeldistillation}]\label{def:lrh}
    Fix a source neural network $f: \mathcal{X} \to \mathcal{Y}$ and let $\varphi: \mathcal{X} \to \mathbb{R}^m$ be the latent representation of $f$. Let $\mathcal{Z}$ be a set of functions $z: \mathcal{X} \to \mathcal{Y}$. For any $\tau > 0$, we say that $f$ satisfies $\tau$-LRH for features $\mathcal{Z}$ if for all $z \in \mathcal{Z}$, there exists $w \in \mathbb{R}^m$ such that $\|w\| \leq \tau$ and $\langle w, \varphi(x) \rangle = z(x)$ for all $x \in \mathcal{X}$.
\end{definition}

In our setting, we think of $f$ as a large foundation model and aim to train a small structured model, specifically a GNN, given minimal access to $f$, under a guaranty that their abstract latent representation\footnote{The penultimate layer of a neural net, or the collection of all its activations, can act as the network's latent representation. Since the latent dimension $m$ appears in the complexity bounds, one should choose the smallest latent representation that still satisfies $\tau$-LRH in order to optimize these bounds.} is rich enough to capture the target ground truths (a combinatorial optimization task) with just another linear layer. 
Surprisingly, we show that distillation is theoretically possible in this setting.

\section{Structured distillation framework}
All graphs in this section are labeled and have $n$ vertices; we denote this collection by $\mathcal{G}$. For notational convenience, we assume $n$ is a power of $2$. 

We begin by giving precise definitions of some concepts informally discussed earlier:
\begin{definition}[Neural networks that compute graph algorithms]
    A neural network $\nu: \mathcal{X}\times \mathcal{G}  \to \mathcal{Y}$ \emph{computes the graph algorithm} $\mathcal{A}$ if $\nu$ agrees with $A$ on all inputs of size $n$. It is \emph{efficient} if it can be evaluated in polynomial time in $n$.
\end{definition}

\begin{definition}[Local-iteration algorithm]\label{def:local_iteration_algorithm}
    Denote by $\mathcal{M}(\{0,1\})$ the set of multisets of elements in $\{0,1\}$. For any multiset-function $g: \mathcal{M}(\{0,1\}) \times \mathcal{G} \times [n] \to \{0,1\}$, let $\mathcal{A}^l[g]:  \{0,1\}^n \times \mathcal{G} \to \{0,1\}$ be the graph-input algorithm that computes \Cref{alg:local_iteration}. 
\end{definition} 

For the remainder of the paper, we assume that the stopping time $l$ is fixed and known \textit{a priori}. In this setting, each vertex has a one-bit hidden representation, which is updated by $g$ at each round of aggregation. Although some of our results extend to multi-bit hidden representations, we focus on the one-bit case for simplicity.

We will consider the following specialization of $\tau$-LRH, which is intuitive in the context of graph algorithms:
\begin{definition}[Local-iteration alignment]\label{def:local_iteration_alignment}
     Fix a source neural network $\nu \in \mathcal{F}$ and let $\varphi: \mathcal{X} \to \mathbb{R}^m$ be the latent representation of $\nu$. Let $Z$ be a set of `key features' $z:\{0,1\}^n \times \mathcal{G} \to \{0,1\}$. For any $\tau > 0$, we say that $\nu$ satisfies $\tau$-local-iteration alignment for $Z$ if for all $z \in {Z}$, there exists a $w \in \mathbb{R}^m$ with $\|w\| \leq \tau$ and $\langle w, \varphi(x) \rangle = \mathcal{A}^l[z](x)$ for all inputs $x$. 
\end{definition}

Lastly, we define decision trees and root-prefix paths:
\begin{definition}
    A decision tree $T: \{0,1\}^d \to \{0,1\}$ is a labeled rooted binary tree with leaves labeled $0$ or $1$ and internal vertices labeled by \textit{literals}\footnote{a \textit{literal} of a Boolean variable $x_i$ is either $x_i$ or $\neg x_i$} of its input variables $x_1 \ldots x_d$. For each input $x \in \{0,1\}^d$, $x$ takes a root-prefix path to arrive at some leaf that specifies the output of the tree on that input.
\end{definition}
\begin{definition}
    In a decision tree $T$ of depth $r$ rooted at $\textsc{root}$, a \textit{root-prefix path} $S$ is a clause, i.e., a tuple of literals, $S =(p_1, \ldots, p_{r'})$ for some $r' \leq r$, where each $p_i \in \{x_1, \ldots, x_d, \neg x_1, \ldots, \neg x_d\}$, such that $p_1 = \textsc{root}$ and $S$ forms a path from the root to some vertex in $T$. The path need not reach a leaf.
\end{definition}

\subsection{Concept class, source class and target class}
We consider the following concept class, i.e., the collection of possible ground truths:
\begin{equation}
    \mathcal{C}_{s,r} = \{\mathcal{A}^l[T] \mid \text{ $T$ is a decision tree with depth $r$ and size $s$}\}
\end{equation}
We address the conditions that ensure that $T$ is a well-defined aggregator, as well as how different choices of $T$ generalize the setting of \Cref{alg:local_iteration}, in the next subsection.

To define the source class, we first specify which features are linearly represented by the source functions. Following \citet{boixadsera2024theorymodeldistillation}, for a decision tree $T$, we take the features of $T$ to be its root-prefix path conjunctions:
$Z'_T := \{\bigwedge_{p_i \in S} p_i \mid S \text{ is a root-prefix path in } T \}$.

The corresponding features for local-iteration algorithms are then
\begin{equation}
    Z_T := \left\{\mathcal{A}^l[b] \mid b \in Z'_T\right\}. \label{eq:key_features}
\end{equation}
We postulate that these loops over prefix-path conjunctions are simply representable by the neural network's latent representation.


Let $\mathcal{F}^\tau_{s,r}$ denote the class of source networks that implicitly compute $\mathcal{A}^l[T]$ for some decision tree $T$ of size $s$ and depth $r$ satisfying $\tau$-local-iteration alignment with features $Z_T$:
\begin{equation}
    \mathcal{F}^\tau_{s,r}
    =
    \{ f \mid \exists T \text{ such that } f \text{ implicitly computes } \mathcal{A}^l[T] \}.
\end{equation}

The target class is the same as the concept class. In practice, this class is a subset of GNNs, so the distillation process can be viewed as distilling learned neural networks into GNNs.
Conceptually, $\mathcal{C}_{s,r}$ is the class of possible ground-truth functions, not the architecture itself. In our setting the two coincide operationally because every element of $\mathcal{C}_{s,r}$ is computed by a local-iteration procedure that can be represented by the aligned GNN target class.

\textbf{On the aggregation function $T$.}\label{subsec:T}
When using a decision tree $T$ as an aggregation function, some well-definedness issues must be addressed. 

First, by \Cref{def:local_iteration_algorithm}, the input to the aggregator $g$ may have variable length, since different vertices can have different numbers of neighbors. Because the graphs are simple, $g$ takes at most $n$ inputs. This can be handled in several ways. For example, one can define $n$ different trees $T^i$, for $i \in [n]$, one for each input length, and pass an additional $\log n$ bits indicating which subtree $T^i$ to use. This increases the depth by at most an additive $\log n$ factor over the maximum depth of the trees $T^i$. 

Second, \Cref{def:local_iteration_algorithm} requires a multiset function, whereas decision trees take ordered tuples as input. In practice, this corresponds to using node-ids to break symmetry, making our problem more general to study from a statistical learning perspective \citep{kiani2024hardness}.

Third, our setup also applies to the more general case in which $T$ is non-local; for example, $T: \{0,1\}^d \times \mathcal{G} \times [n]$ may take as input the full list of hidden representations $h$, together with the current node-id. In this global setting, one can decouple $d$ from $n$ by allowing \Cref{alg:local_iteration} to include extra input bits that are passed through to $T$. This lets us fix $n$ without also fixing the complexity parameters of the decision tree $T$.

Finally, by the law of currying, our set-up in \Cref{alg:local_iteration} works  even without parameter sharing restriction. In other words, we can have a separate tree $T_v$ for each vertex $v$, and compose these per-vertex trees with a vertex selector tree of size $n$ and depth $\log n$ (\Cref{fig:vertex_selector}). Therefore, although using bounded depth (and size) $T$ as an aggregator loses some expressivity compared to using a neural network, our setting allows for other generalizations, namely symmetry breaking with node-id and the use of per-vertex aggregators. 

\textbf{On the key-features $Z_T$.}
The `simple' features that are linearly representable by the source network (\Cref{eq:key_features}) take the form of the template $\mathcal{A}^l$ applied to a conjunction $\bigwedge_{i = 1}^{s'} p_i$ for some root-prefix path $p = (p_1,\ldots,p_{s'})$. A conjunction can be viewed as a degenerate decision tree that is a single path: it outputs $1$ exactly when all literals on the path are satisfied. By our parameterization in the previous remark, if $s' \leq \log n$, then the hidden representations $h$ are all a partition of the vertex set; otherwise, the feature checks a specific conjunction of bits in the previous hidden representation. {Such functions are expected to be learnable by our source class in the aligned regime, and we experimentally test this in \Cref{subsec:experiments} `Validating the LRH'. }

\section{Efficient distillation}
\subsection{GNNs are more efficient than decision trees}
The following simple fact gives a separation between GNNs and decision trees:
\begin{lemma}\label{lem:gnn_vs_dt}
    There exists a simple decision tree $T : \{0,1\}^{n} \times \{0,1\}^{\binom{n}{2}} \to \{0,1\}$ that can be evaluated in polynomial time in $n$ such that, for some constant $l$, $\mathcal{A}^l[T]$ can be evaluated in polynomial time, but it cannot be represented by decision trees of polynomial size.
\end{lemma}
The proof considers the $2$-reachability DP: Given a graph on $n \geq 2$ vertices, is there a path of length at most $2$ that connects the vertex labeled $1$ and $n$? The full proof is in Appendix~\ref{app:proof_of_gnn_vs_dt}. As a result, without the for-loop structure, one cannot just convert the concept class $\mathcal{C}_{s,r}$ into the class of efficient decision trees. This forms a type of algorithmic mismatch which we resolved in the next section, using GNNs. 

\subsection{Distillation under $\tau$-LRH}

We are now ready to state the main theorem.
\begin{theorem}\label{thm:efficient_distillation}
For any $\epsilon, \delta \in (0,1)$, there is an algorithm that $(\epsilon, \delta)$-distills from $\mathcal{F}^\tau_{s,r}$ to $\mathcal{C}_{s,r}$ and runs in time polynomial in $s^n, m, 1/\epsilon, d, 2^{nrl}, \log(1/\delta), \tau, B$ and uses a number of source-model samples polynomial in $1/\epsilon, s, n \log(d/\delta), \log(\tau B)$, where $m$ is the latent representation dimension and $B = \max_x \|\varphi(x)\|$. 
\end{theorem}
\textbf{Efficiency.}
    The sample complexity analysis of the algorithm follows from a Hoeffding bound and is polynomial in the stated parameters. The time complexity should be interpreted more narrowly. The algorithm is efficient only in the restricted regime where (i) the number of GNN iterations $l$ is a constant, (ii) the depths of the true decision trees are logarithmic or otherwise small, and (iii) the number of vertices $n$ is fixed. Thus the theorem is a complexity-theoretic positive result for a structured setting, not a claim of practical scalability for arbitrary graph sizes or combinatorial instances. This restriction is nevertheless meaningful in comparison with a naive reduction to ordinary decision-tree distillation: unrolling the local-iteration computation and then applying \Cref{thm:dt_distillation} would lead to a dependence exponential in the depth of the unrolled tree, on the order of $2^{r^l}$ in the worst case. Our algorithm instead uses the local iterative structure directly. Removing the fixed-$n$ dependence, or reducing it under symmetry, parameter sharing, or equivariance constraints, is an important open direction; such restrictions are also motivated by evidence that learning under invariances can yield exponential savings \citep{kiani2024hardness}.

\textbf{Improvement over~\citep{boixadsera2024theorymodeldistillation}.}
    Comparatively, a naive way to use the decision-tree distillation algorithm of \citet{boixadsera2024theorymodeldistillation} in our setting is to first unroll the entire $l$-step local-iteration algorithm into a single ordinary decision tree, and then apply their algorithm to this flattened tree. Indeed, if the inner decision tree has depth $r$, each rounds of unrolling adds $r$ depths per node to the unrolled tree. Thus the fully unrolled computation  have depth $O(r^l)$ and their algorithm naively runs in time $2^{O(r^\ell)}$, which is not polynomial even with our restricted setting. We exploit the iterative structure directly in the proof to circumvent this issue. 

In the remainder of this section, we describe a GNN distillation that admits the complexity guarantees in Theorem~\ref{thm:efficient_distillation}. A full proof of the theorem can be found in
Appendix~\ref{app:proof_of_efficient_distillation}.

\paragraph{GNN distillation algorithm}
    Algorithm~\ref{alg:distillation}, operates in two phases. It first builds a set of paths (conjunctions) that is a superset of all root-prefix paths in the true tree $T$, with high probability. This is done using LRH-type assumptions to linearly probe which paths would be likely to appear in the true decision tree. Then it stitches these paths together efficiently using a modification of a classical tree building DP~\citep{mehta2002dp}.

\begin{figure}[t]
    \centering
    \begin{minipage}[t]{0.39\linewidth}
        \centering
        \includegraphics[width=\linewidth]{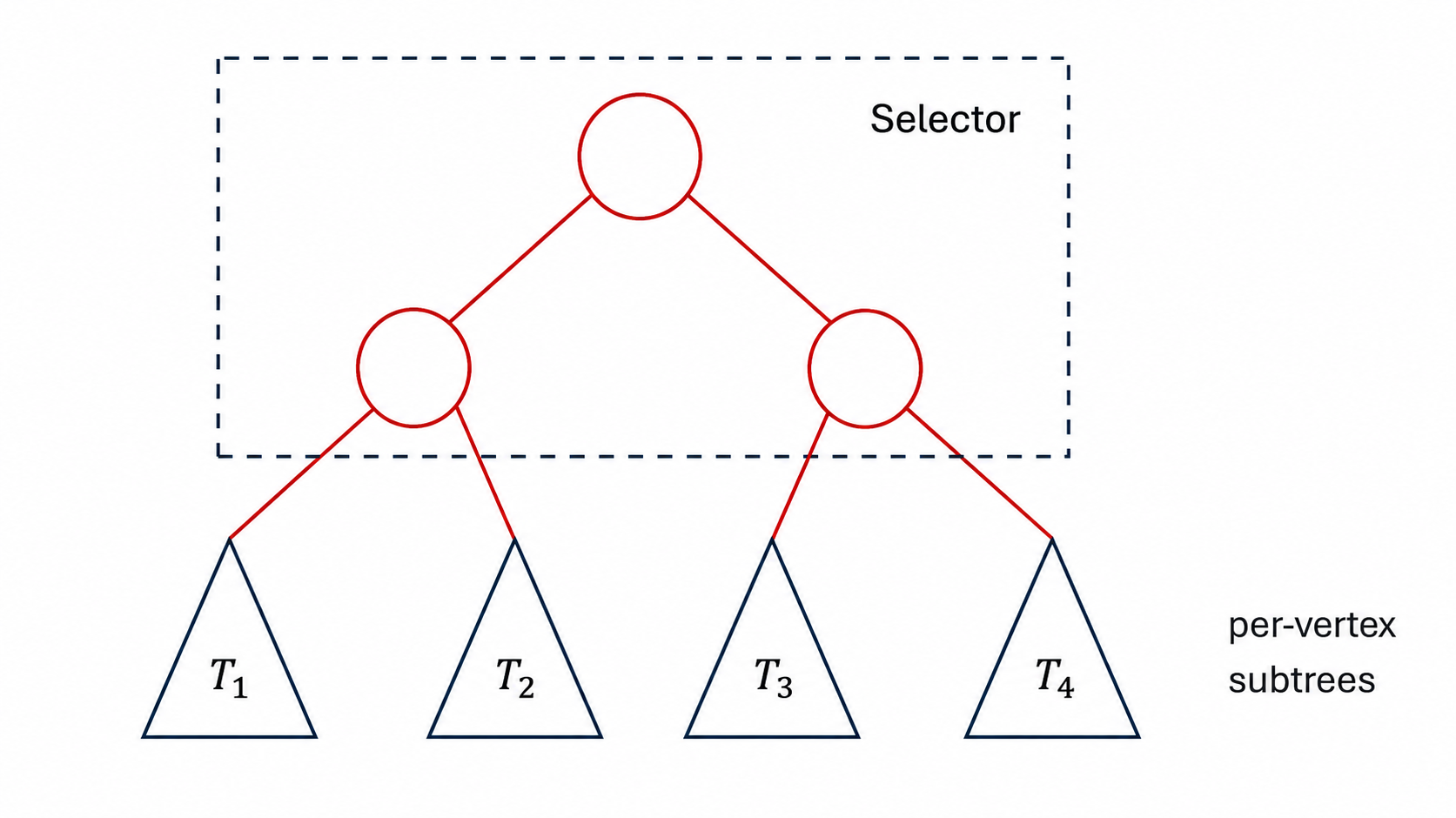}
        \caption{The global decision tree $T$ can be decomposed into a vertex {Selector}, followed by per-vertex subtrees at its leaves. This generalizes traditional GNNs by allowing for different aggregation schemes based on node statistics or id.}
        \label{fig:vertex_selector}
    \end{minipage}
    \hfill
    \begin{minipage}[t]{0.57\linewidth}
        \centering
        \includegraphics[width=\linewidth]{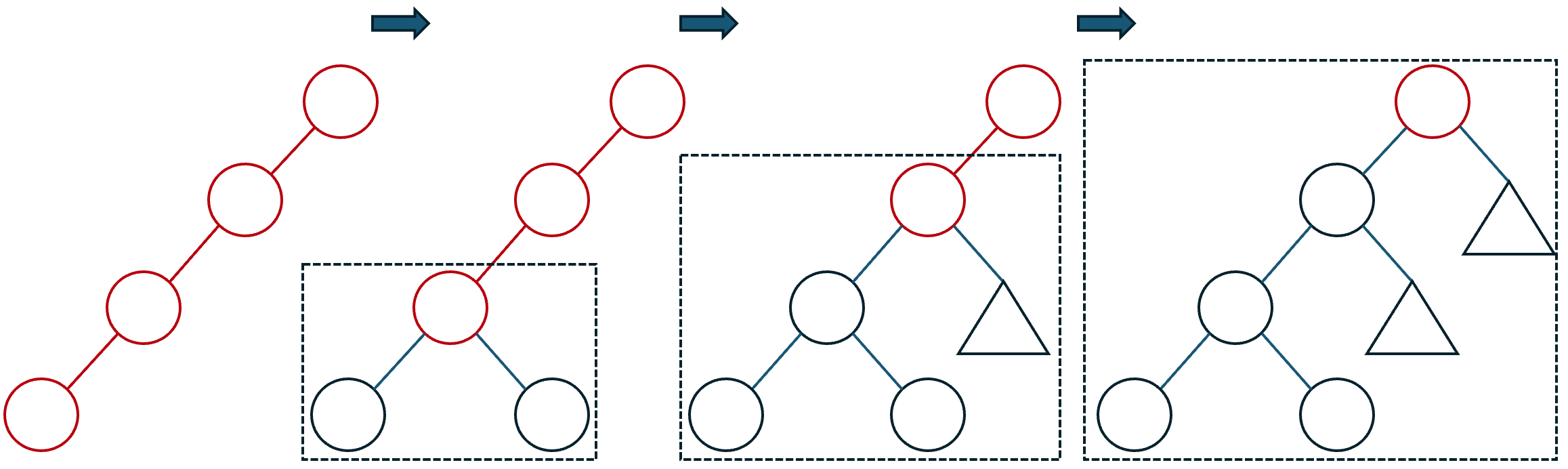}
        \caption{The dynamic program that optimizes for the inner decision in the second phase of \Cref{alg:distillation}. Circles represent decision-tree nodes while triangles represent subtrees. The base cases are root-prefix path candidates collected in the first phase; the DP then optimizes subtrees of increasing sizes rooted at the ends of candidate root-prefix paths.}
        \label{fig:tree_building_dp}
    \end{minipage}
\end{figure}

\begin{algorithm}
\caption{GNN distillation algorithm}
\label{alg:distillation}
\textbf{Input} Neural network $\nu$, random samples from $\mathcal{D}$, depth bound $R \in \mathbb{N}$, error parameters $\epsilon, \delta > 0$\\
\textbf{Output} A GNN that computes $\mathcal{A}^l[\widehat{T}]$

\begin{algorithmic}[1]
\State \textit{\slash *  Phase 1: Collecting root-prefix paths *\slash}
\State $\mathcal{S}_0 \gets \emptyset$
\For{$u = 1$ to $n$}
  \State $\mathcal{S}_0 \gets \mathcal{S}_0 \cup \{p\mid p \text{ is a root-prefix path of } \mathrm{Selector} \text{ leading to } u\}$
\EndFor

\For{$i = 1$ to $R$}
  \State 
    $\mathcal{P}_{i-1} \gets
  \Bigg\{\, S \in \mathcal{S}_{i-1} \Big\vert
  \textsc{LinearProbe}\!\Bigg(
    \mathcal{A}^l\!\left[\bigwedge_{p \in S} p\right],
    \varphi, B, \tau, 2^{-il-3}, \frac{\delta}{2|\mathcal{S}_{i-1}|R}
  \Bigg)
  =\textup{true}
  \Bigg\}$

  \State $\mathcal{S}_i \gets
  \displaystyle \bigcup_{S \in \mathcal{P}_{i-1}}
  \bigcup_{j=1}^d
  \left\{ S \cup \{x_j\},\; S \cup \{\neg x_j\} \right\}$
\EndFor

\State $\mathcal{S} \gets \displaystyle \bigcup_{j=0}^R \mathcal{S}_j$

\vspace{2em}
\State \textit{\slash * Phase 2: Estimate the true tree using paths from $\mathcal{S}$ *\slash}
\State $\mathcal{S} := \bigsqcup_{u = 0}^n \mathcal{S}^u$ decomposing paths by the per-vertex subtree $T_u$ identified by their selector prefix
\ForAll{$S_1 \in \mathcal{S}^1, \ldots, S_n \in \mathcal{S}^n$}
  \State $\hat{v}_{S_1, \ldots, S_n} \gets
  \mathbb{E}_{y \sim \mathcal{D}}\!\left[
    \left( \prod_{u = 1}^n \prod_{p \in S_u} p(y)\right)
    (2\nu(y)-1)
  \right] \pm \epsilon/|\mathcal{S}|$
\EndFor

\State \textbf{return}  $\arg\max_{\widetilde{T}_1, \ldots, \widetilde{T}_n} \text{val}(\widetilde{T}_1, \ldots, \widetilde{T}_n,\hat{v})$, where $\widetilde{T}_i$ ranges over decision trees with
$
Z'_{\widetilde{T}_i} \subseteq \left\{ \bigwedge_{p \in S} p \ \middle|\ S \in \mathcal{S}^i \right\}.
$

\end{algorithmic}
\end{algorithm}

The algorithm has two key subroutines. The $\textsc{LinearProbe}(q, \varphi, B, \tau, \epsilon, \delta, \mathcal{D})$ subroutine comes from Lemma 3.7 of \citep{boixadsera2024theorymodeldistillation}. It runs in polynomial time and draws polynomially many samples. Intuitively, it tests whether the candidate Boolean function $q$ is easy to read out linearly from the source representation $\varphi$. It returns \textup{true} w.p. $\geq 1 - \delta$ if there is a $w \in \mathbb{R}^m$ with $\|w\| \leq \tau$ and $\E_x[(\langle w, \varphi \rangle - q)^2] \leq 2\epsilon$, and returns \textup{false} w.p. $\geq 1 - \delta$ if for all such $w$, the expectation is at least $2 \epsilon$. Maximization over decision trees in the final step is done with a DP modified from \citep{guijarro1999dp,mehta2002dp}. After Phase 1, we know a pool of plausible root-prefix path fragments, but not how to connect them into a tree. This DP stitches the fragments into the tree that optimizes the objective $\textup{val}$; the objective is chosen so that maximizing it corresponds to minimizing the $0$-$1$ risk of the candidate local-iteration algorithm. In details:

\textbf{Phase 1.} We start by initializing the collection $\mathcal{S}$ with root-prefix paths of the vertex selector tree $\mathrm{selector}$ of size $n$ and depth $\log n$ (as discussed in \Cref{subsec:T}). The loop in Lines 2--4 adds the selector prefixes that identify which per-vertex subtree $T_u$ is used. Since the selector decides which of the $n$ per-vertex subtrees $T_1,\ldots,T_n$ processes the input, each such prefix is a root-prefix path of the global true tree $T$. After this initialization, we recursively grow each surviving path by one valid literal at a time. For a candidate path $S$, the $\bigwedge_{p \in S}p$ computes  the conjunction of all literals in $S$. If the corresponding local-iteration feature $\mathcal{A}^l[\bigwedge_{p\in S}p]$ cannot be linearly represented by the source network with a low-norm coefficient vector, as checked by \textsc{LinearProbe}, then the path is pruned. The key technical analysis is to prove both sides of this filtering step: all true root-prefix paths pass with high probability, while not too many spurious paths can pass. The latter is a packing argument showing that the source representation cannot simultaneously contain too many low-norm linear readouts for unrelated candidate Boolean functions.

\textbf{Phase 2.} We set up a DP algorithm that builds up our estimated inner decision tree. This step is more involved than that of \citep{boixadsera2024theorymodeldistillation}, because the same input to $\mathcal{A}^l[\widehat{T}]$ can traverse the global tree $\widehat{T}$ along different paths depending on which vertex $u$ is being updated. 
To resolve this issue, we build each of the $n$ per-vertex subtrees $T_1, \ldots, T_n$ simultaneously. 
Finally, before setting up the DP objective, recall that each root-prefix path $S$ of size more than $\log n$ contains a $\log n$-length selector prefix that identifies which subtree $T_u$ the path comes from. Thus, Line 11 decomposes $\mathcal{S} := \bigsqcup_{u = 0}^n \mathcal{S}^u$, where $\mathcal{S}^0$ contains the selector-only paths and $\mathcal{S}^u$ contains paths from per-vertex subtree $T_u$. Lines 12--14 then estimate local correlation statistics $\hat v_{S_1,\ldots,S_n}$ for tuples of retained paths, one from each candidate set. The tree-building DP uses these statistics to jointly infer all per-vertex trees.

\textbf{Valuation function.} 
Now we can set up the DP objective. For a candidate tree $\widetilde{T}$  of depth $r$ and size $s$ with per-vertex subtrees $\widetilde{T}_i$ for each $i \in [n]$ and for any weight function $u: \bigotimes_i \mathcal{S}_i \to \mathbb{R}$, define the valuation function:
\begin{equation}
    \text{val}(\widetilde{T}_1, \ldots, \widetilde{T}_n, u) := \sum_{S_i \in \text{Leaves}(T_i), i \in [n]} (2\widetilde{T}(S_1,\ldots,S_n) -1) u_{S_1,\ldots,S_n},
\end{equation}
where $\widetilde{T}(S_1,\ldots,S_n)$ is computes $\mathcal{A}[\widetilde{T}]$ after replacing $\widetilde{T}_i$ with $S_i$ only in the first layer.  
Using $u := v$--the exact expectation in Line 15, we claim to get:
$
    \text{val}(\widetilde{T}_1, \ldots, \widetilde{T}_n, v) = \mathbb{E}_{x}\left[(2\nu(x) - 1)(2\widetilde{T}(x) - 1)\right]
$, which negatively correlates with the $0$-$1$ loss for the estimated trees. In \Cref{alg:distillation}, one uses a Hoeffding bound to obtain $\hat{v}$ as an empirical approximation to $v$.
We can now describe the DP that estimates our tree. The states of the DP are indexed by tuples $(S'_1, \ldots, S'_n, s'_1, \ldots, s'_n) \in \mathcal{S}^n \times [s]^n$ where $s$ is the size upper bound for the true decision tree. At state $(S'_i,s'_i)$, the transition computes the optimal subtree of size $s'_i$ rooted at the end of path $S'_i$ for each $i \in [n]$. The tree that optimizes $0$-$1$ loss can then be queried at the final DP state. A simplification of this procedure is shown in \Cref{fig:tree_building_dp}.


\section{Experiments}\label{subsec:experiments}

We fix the number of vertices to be $n=6$ and the number of message passing rounds to be $l = 6$. For each depth $r \in \{2,3,4,5\}$, we generate an independent uniformly random decision tree $T_u$ of depth $r$ for each vertex $u$. These per-vertex trees, together with a fixed vertex selector, define the ground truth local-iteration algorithm $\mathcal{A}^l[T_1,\ldots,T_n]$. We use a 5-layer ResNet with width $1000$ as our large unstructured source model and train it on uniformly sampled Boolean graph inputs under classification loss using stochastic gradient descent. We choose these dimensions so that the largest task in the experiment, depth-$5$ per-vertex trees, can be fit by the source model. 


\textbf{Validating the LRH.} After training the source model, we collect all hidden activations as the representation $\varphi$. For every root-prefix conjunction $b$ appearing in the true per-vertex trees, we train a linear probe $w$ so that $\langle w,\varphi\rangle$ approximates $\mathcal{A}^l[b]$. The probe is trained for $100$ iterations with Adam on squared loss using uniformly sampled graph inputs. To test the low-norm form of the assumption, we also run projected gradient descent with a prescribed norm bound on $w$. \Cref{tab:lrh_evidence} reports the resulting train and test errors. Low error under the norm constraint is evidence that the trained source model satisfies the local-iteration alignment condition used by 

\textbf{End-to-end distillation.} We next evaluate the full algorithmic pipeline by implementing a heuristics of \Cref{alg:distillation}. For each depth $r \in \{2,3,4,5\}$ and each probe number $k \in \{10,50,100,200\}$, Phase 1 starts from the vertex-selector prefixes and probes candidate clauses of the form $\mathcal{A}^l[\wedge_{p\in S}p]$. We use a practical top-$k$ probe search: at each depth, we keep the $k$ clauses with lowest validation probe error and branch only from those clauses. Phase 2 then ranks the collected clauses by probe error, keeps a bounded number per vertex, and runs the tree-building DP over this pruned set. Because explicitly materializing $\hat v_{S_1,\ldots,S_n}$ has size $\prod_u |\mathcal{S}^u|$, which is already prohibitive for $n=6$, the implementation estimates the phase-2 objective by Monte Carlo agreement with the source network. We report the accuracy of the pipeline in \Cref{tab:algo2_sweep} and some more detailed diagnostics in the Appendix (\Cref{tab:algo2_phase2_diagnostics}). 

\begin{table*}[t]
    \centering
    \begin{minipage}[t]{0.48\linewidth}
        \centering
        \resizebox{\linewidth}{!}{%
        \begin{tabular}{ccccc}
     \begin{tabular}{c} depth  \\ (norm) \end{tabular}& \begin{tabular}{c} \# of \\ conj. \end{tabular} & source acc. &  \begin{tabular}{c} average \\ training err \end{tabular} &  \begin{tabular}{c} average \\ testing err \end{tabular}\\ 
    \hline
    2 ($\infty$) & 42 & $1.000$ & $0.1982$ & $0.2058$ \\
    2 ($0.001$) & 42 & $1.000$ & $1.898$ & $1.935$ \\
    \hline
    3 ($\infty$)& 90 & $1.000$ & $0.3155$ & $0.3298$ \\
    3 ($0.001$)& 90 & $1.000$ & $1.258$ & $1.270$ \\
    \hline
    4 ($\infty$)& 186 & $0.885$ & $0.0875$ & $0.0971$ \\
    4 ($0.001$)& 186 & $0.885$ & $0.4680$ & $0.4690$ \\
    \hline
    5 ($\infty$)& 378 & $0.784$ & $0.0526$ & $0.0602$ \\
    5 ($0.001$)& 378 & $0.784$ & $0.2164$ & $0.2177$ \\
    \end{tabular}%
        }
        \caption{Linearly probing $\mathcal{A}^l[b]$ where $b$ is a single conjunction in the true tree. First column includes the true depth, and a norm upper bound for projected descent.}
        \label{tab:lrh_evidence}
    \end{minipage}
    \hfill
    \begin{minipage}[t]{0.48\linewidth}
        \centering
        \resizebox{\linewidth}{!}{%
        \begin{tabular}{cccccc}
    depth & $k$ & source acc. & distil. acc. & \# probes & probe frac.\\
    \hline
    2 & 10 & $1.000$ & $0.754$ & $360$ & $0.208$\\
    2 & 50 & $1.000$ & $0.745$ & $1041$ & $0.600$\\
    2 & 100 & $1.000$ & $0.736$ & $1508$ & $0.870$\\
    2 & 200 & $1.000$ & $0.757$ & $1734$ & $1.000$\\
    \hline
    3 & 10 & $1.000$ & $0.909$ & $557$ & $0.045$\\
    3 & 50 & $1.000$ & $0.898$ & $2002$ & $0.163$\\
    3 & 100 & $1.000$ & $0.911$ & $3355$ & $0.273$\\
    3 & 200 & $1.000$ & $0.900$ & $5198$ & $0.423$\\
    \hline
    4 & 10 & $0.881$ & $0.652$ & $736$ & $0.012$\\
    4 & 50 & $0.881$ & $0.646$ & $2856$ & $0.048$\\
    4 & 100 & $0.881$ & $0.634$ & $5096$ & $0.085$\\
    4 & 200 & $0.881$ & $0.679$ & $8642$ & $0.144$\\
    \hline
    5 & 10 & $0.791$ & $0.667$ & $894$ & $0.004$\\
    5 & 50 & $0.791$ & $0.686$ & $3665$ & $0.017$\\
    5 & 100 & $0.791$ & $0.661$ & $6649$ & $0.031$\\
    5 & 200 & $0.791$ & $0.675$ & $11729$ & $0.055$\\
    \end{tabular}%
        }
        \caption{End-to-end Algorithm \ref{alg:distillation}. Here $k$ is the top-$k$ probe size used in Phase 1.}
        \label{tab:algo2_sweep}
    \end{minipage}
\end{table*}



\section{Conclusion and discussion}

In this paper, we extend the work of \citet{boixadsera2024theorymodeldistillation} in PAC-distillation to study the well-known  `learn first, distill later' paradigm, where a large multipurpose model is trained on a variety of tasks, then distilled to more structured models that have built-in algorithmic alignment properties, such as GNNs for learning DP algorithms. We showed that although some DP algorithms have a local transition rule representable by a small decision tree, the full DP computation cannot itself be represented by an efficient decision tree -- a case of misalignment.  On the other hand, the local iteration structure of a GNN with decision tree aggregation allows for distillation from a large, learned neural network that exhibits a certain kind of linear representability. We also propose an algorithm that is tractable when the number of GNN iterations is fixed, the depth of the inner decision tree is small, and the number of vertices in the input graphs is fixed.

Several exciting questions and future directions arise from this work. We conjecture that in a setting as general as ours (as described in \Cref{subsec:T}), an exponential lower bound in the number of vertices $n$ could be proven. One can also study more restricted settings, such as shared transition trees or equivariant parameterizations, to obtain a better dependence on $n$. 
Finally, one could analyze distillation algorithms that are popular in practice, such as reinforcement learning or teacher-student supervised learning.

\section*{Acknowledgments}
    This research was developed with funding from the Defense Advanced Research Projects Agency (DARPA) under agreement no. HR0011-25-3-0205. The views, opinions, and/or findings expressed are those of the authors and should not be interpreted as representing the official views or policies of the Department of Defense or the U.S. Government. MW acknowledges partial support from an Alfred P. Sloan Fellowship in Mathematics and the AI2050 program at Schmidt Sciences (Grant G-25-69786).

\bibliographystyle{plainnat}
\bibliography{main}

\clearpage
\appendix

\section{Extended Background}\label{app:extended_background}

\paragraph{Graph machine learning}
Graph machine learning is a testbed for graph-based inductive biases that may allow for exponential gains in learning efficiency. Informally, symmetry constraints of graph functions, in terms of vertex permutations, induce certain sparsity structures in the function space, making learning more data-efficient~\citep{bietti2021gain,elesedy2021gain,tahmasebi2023gain}. However, learning graph neural networks and other equivariant networks is still computationally hard in the worst case, requiring, for example, exponentially or superpolynomially many queries in the correlation statistical queries model of learning \citep{kiani2024hardness}. Understanding which settings exactly give rise to quantitative benefits for learning is an important and active area of research.

More specifically for graphs, a graph neural network (GNN) \citep{gilmer2017gnn,kipf2017semisupervised} is a deep-learning parameterization of the space of functions on graphs, potentially of different sizes.

\paragraph{Graph neural networks} The main architecture we consider as the target class is that of graph neural networks with strong inductive bias for graph datasets. In particular, we are interested in message passing neural network (MPNN) (\Cref{def:mpnn}), in which each node aggregates neighboring information and processes them with a neural network to form a new latent representation in each round. After a fixed number of rounds, the network outputs a learned representation for each vertex of the graph, or combines them together to form a single representation for the whole graph, depending on the specific tasks. While we will only consider decision tree aggregators, neural networks can (arguably efficiently) emulate decision trees, since they are universal approximators. 

\begin{definition}[Message-passing neural network]\label{def:mpnn}
Let $G=(V,E)$ be a graph with node features $x_v \in \mathcal X$ for $v\in V$.
An $l$-layer message-passing neural network (MPNN) consists of an initialization map
$\iota:\mathcal X\to \mathbb R^{d_0}$, message maps
$M_t:\mathbb R^{d_{t-1}}\times \mathbb R^{d_{t-1}}\to \mathbb R^{q_t}$,
update maps $U_t:\mathbb R^{d_{t-1}}\times \mathbb R^{q_t}\to \mathbb R^{d_t}$,
and a permutation-invariant aggregation operator $\mathrm{AGG}_t$ on multisets.
It computes hidden states
\[
h_v^{(0)}=\iota(x_v),\qquad
m_v^{(t)}=\mathrm{AGG}_t\bigl(\{M_t(h_v^{(t-1)},h_u^{(t-1)}):u\in \mathcal N(v)\}\bigr),
\]
\[
h_v^{(t)}=U_t(h_v^{(t-1)},m_v^{(t)}),\qquad t=1,\ldots,l.
\]
A node-level MPNN outputs $\rho_{\mathrm{node}}(h_v^{(L)})$, while a graph-level MPNN outputs
\[
\rho_{\mathrm{graph}}\bigl(\mathrm{READOUT}(\{h_v^{(L)}:v\in V\})\bigr),
\]
where $\mathrm{READOUT}$ is permutation invariant. If the message, update, and readout maps are neural networks, we call the resulting architecture an MPNN.
\end{definition}

\paragraph{Combinatorial optimization with graph ML}
One proposed area where GNNs could have strong inductive bias with the learning task is that of using neural networks to learn combinatorial optimization. It is observed \citep{xu2020algorithmicalignment} that the loop structure of an MPNN closely follows that of local graph algorithms, such as Bellman-Ford for shortest path. As such, \citet{xu2020algorithmicalignment} argues that the neural network used in the aggregation operation of an MPNN only had to learn a simple function of its inputs, and not the actual for-loop structure, thus decreasing the sample complexity of learning from supervised examples produced by such algorithms.  Although the original paper provided a theoretical justification for this phenomenon through PAC learning \citep{valiant1984pac}, a tighter analysis of what constitutes such \emph{algorithmic alignment} has drawn many follow-up investigations \citep{dudzik2022gnnsdp, dudzik2024cocycles}. Nevertheless, the idea that learning architecture should be built to resemble a potential algorithmic paradigm, such as dynamic programming, is intuitive and has been the inspiration for many neural heuristics that are widely successful in practice \citep{kahng2024steiner,nerem2025graphneuralnetworksextrapolate,he2025primaldual,gasse2019co}.

\paragraph{PAC-distillation keystone result}
To give a taste of the results that can be obtained from this framework, we restate a result from~\citet{boixadsera2024theorymodeldistillation}. Recall that in the Boolean setting, a decision tree has vertices labeled by some literals of the input bits and each vertex is only reachable by  inputs that satisfy the conjunction of literals on the path from the root to said vertex.

\begin{theorem}[Theorem 3.6 of \citep{boixadsera2024theorymodeldistillation}]\label{thm:dt_distillation}
   Let $\mathcal{F}$ be the set of neural networks $f$ that implicitly compute a decision tree $T:\{0,1\}^d \to \{0,1\}$ of depth $r$ and size $s$ such that $f$ satisfies $\tau$-LRH for features $\mathcal{Z}_T := \{\bigwedge_{p \in S} p : \text{$S$ is a path of literals from the root of $T$ to any of its vertices}\}$. Let $\mathcal{H}$ be the set of decision trees with depth $r$ and size $s$. Then for any $\epsilon, \delta \in (0,1)$, there is an algorithm that $(\eps,\delta)$-distills from $\mathcal{F}$ to $\mathcal{H}$ that runs in polynomial time in $d, m, 1/\epsilon, s, 2^r, \log(1/\delta), \tau$ and $B$ and takes polynomially many samples in $1/\epsilon, s, \log(d/\delta), \log(\tau B)$ where $B \geq \max_x \|\varphi(x)\|$. 
\end{theorem}
\begin{remark}\label{rmk:dt_distillation_vs_learning}
This is an unexpected result, as it is unknown if PAC-learning a decision tree can take less than $d^{O(r)}$ time \citep{weisberg2020distributionfree}
. On the other hand \Cref{thm:dt_distillation} shows that PAC-distillation takes only $\textup{poly}(d, 2^r)$ time.  
\end{remark}

\section{Proof of Lemma \ref{lem:gnn_vs_dt}}
\label{app:proof_of_gnn_vs_dt}
\begin{proof}
    Consider the combinatorial problem of deciding, for a labeled graph, whether the first and last vertex is connected with a path of length at most $2$, or $2$-reachability.

    The classic dynamic programming (DP) algorithm for this problem runs in time $O(n)$ \citep{tarjan1971reachability}. 

    However, any decision tree that correctly solves this problem on all labeled graphs of size $n$ must have exponential size. To see this, we bound the number of leaves of a correct tree (which in turn bounds the order of its size since a decision tree is binary). 
    
    Consider the subset of graphs on the $n$ vertices labeled by $[n]$ where the only possible edges are $(1,n)$ and $(1, v)$, $(v,n)$ for all $v \in V \backslash\{1,n\}$. There are $2^{2(n -2)+1}$ such graphs. Among them, graphs that fail to have a path of size at most $2$ between $1$ and $n$ does not have the $(1,n)$ edge and for each other $v$, have one of the $3$ configurations out of $4$ possible choices of presence/absence of the pair $(1, v)$, $(v,n)$. This counts to $(3/4)^{n-2}/2$ fraction of the total number of graphs. 

    Now, each $0$-leaf (leaf that outputs $0$ for the DT) of a correct DT on these inputs fixes a certain presence/absence of some edges on the path from the DT's root to it. Once certain variables are fixed, all other variables are free to range between $0$ and $1$ and the output of the DT is still $0$. This means that $(1,n)$ must always be included in the fixed variables, and so is at least one in each pair $(1, v)$, $(v,n)$. Thus, each $0$-leaf accounts for at most a fraction of $2^{-(n -1)}$ of the total number of graphs.

    Therefore, the number of leaves must be at least $(3/4)^{n-2}/2 / 2^{-(n -1)}$, which is exponential in $n$.
\end{proof}

\section{Proof of \Cref{thm:efficient_distillation}}\label{app:proof_of_efficient_distillation}
\subsection{Notations and definitions}
We will set up some notation for this particular proof and also remind the readers of previously defined notation:
\begin{itemize}
    \item \textbf{The input space:} Since we are using a graph neural network to emulate an algorithm of the form $\mathcal{A}^l[T]$ for some decision tree $T$, there is a difference between the input of $\mathcal{A}^l[T]$ and that of $T$. The former takes as input an initialization feature in $\{0,1\}^n$ a graph adjacency matrix $\{0,1\}^{n \times n}$ and we write $\mathcal{X}_\mathcal{A}$ for this input space. $T$ itself has an input consisting of the previous representation of each layer $\{0,1\}^{n}$, a graph adjacency matrix $\{0,1\}^{n \times n}$ and additionally the index of a vertex $v$ and we reserve $\mathcal{X} := \{0,1\}^d$ where $d = \Omega(n^2)$. 
    \item \textbf{Logical notation:} for some input bit $x_j$, $j \in [d]$, a \textit{literal} is $x_j$ or its negation $\neg x_j$. A \textit{clause} $S = (p_1, \ldots, p_s)$ is an ordered tuple of $s$ literals and we define $\textsc{AND}_S(x) := \bigwedge_{p \in S} p$ the conjunction of literals in $S$. A \textit{non-degenerate} $k$-clause is a clause $S$ such that $|S| = k$ and each variable appears at most once in $S$. 
    \item \textbf{Decision trees:} Given a decision tree $T$, recall that $Z'_T$ is the set of clauses each corresponds to a path in $T$ with one end-point being the root (i.e. a \textit{root-prefix path}). We include the trivial path $\emptyset$ with $\textsc{AND}_\emptyset = \textsc{true}$ in this collection. We also denote by $Z_T$ the collection of all $\mathcal{A}^l[\textsc{AND}_S]$ functions for each $S \in Z'_T$.
\end{itemize}

In our algorithm, which is an extension of that in \citep{boixadsera2024theorymodeldistillation}, we use some subroutines from the original paper. 
\begin{lemma}[Lemma 3.7 \citep{boixadsera2024theorymodeldistillation}]\label{lem:linear_probe}
    Given a function $g: \mathcal{X} \to [-1,1]$, a representation map $\varphi: \mathcal{X} \to \mathbb{R}^m$ with norm bounded by $B \geq \max_x \|\varphi(x)\|$, $\tau, \epsilon, \delta > 0$ and an input distribution $\mathcal{D}$, there is a subroutine $\textsc{LinearProbe}(g,\varphi,B,\tau,\epsilon,\delta, \mathcal{D})$ that runs in time $\text{poly}(1/\epsilon, \log(1/\delta), \tau, B, m)$ and draws $\text{poly}(1/\epsilon, \log(1/\delta), \tau, B)$ samples from $\mathcal{D}$ such that:
    \begin{itemize}
        \item If there is a $w \in \mathbb{R}^m$ with $\|w\| \leq \tau$ and $\E[(w\cdot \varphi(x) - g(x))^2] \leq \epsilon$, then \textsc{LinearProbe} returns true with probability $1-\delta$.
        \item If there is no $w \in \mathbb{R}^m$ with $\|w\| \leq \tau$ and $\E[(w\cdot \varphi(x) - g(x))^2] \leq \epsilon$, then \textsc{LinearProbe} returns false with probability $1-\delta$.
    \end{itemize}
\end{lemma}

\subsection{Proof of \Cref{thm:efficient_distillation}}
We are now ready to start the proof. Assume that there is a true decision tree $T$. We first show that the paths collection from the distillation algorithm contains all root prefix paths in the true tree: $\mathcal{S} \supseteq Z_T$ with high probability.

\paragraph{$\mathcal{S}$ contains all root-prefix paths of the true tree} From the guarantees of \Cref{lem:linear_probe}, it suffices to show that any clause $S \in Z'_T$ is checked by \textsc{LinearProbe} and thus added to $\mathcal{S}$ with high probability. Assume to the contrary that this is not true; in other words, there is a root-prefix path $S \in Z'_T$ of length $i$ that was not added to $\mathcal{S}_i$. Recall that whenever \textsc{LinearProbe} accepted a clause $S'$, we added all possible extension of $S'$ to our collection $\mathcal{S}$. The fact that $S$ was not included means that either it was not checked by \textsc{LinearProbe} or it was checked and then rejected. In the former case, this means that $S \neq \emptyset$ (since the $\emptyset$ is always checked) and the root-prefix path corresponds to $S$' parent was not included in the previous set $\mathcal{S}_{i - 1}$. We can then use induction on this parent node instead. In the latter case, $S$ was checked but \textsc{LinearProbe} returns false, which occurs with probability $1 - \frac{\delta}{2|\mathcal{S}_i|R}$ because of our linear representation hypothesis and \Cref{lem:linear_probe}.

A union bound at each layer suffices to argue that all length $i$ clauses in $Z'_T$ are in $\mathcal{S}$ with probability $1 - \frac{\delta}{2R}$ for each $i$. Another union bound over all $i$ then concludes that $Z'_T \in \mathcal{S}$ with probability $1 - \delta/2$. 

Now we argue that $|S_i| \leq \text{poly}(2^{\Theta(il)}, \tau, B, d)$
\paragraph{$\mathcal{S}_i$ size is upper-bounded} The key way we control $\mathcal{S}_i$ size is by arguing that only the $\mathcal{A}^l[\textsc{AND}_S]$'s that can be linearly represented by the source network are kept while the rest are pruned. Furthermore, there cannot be too many of these $\mathcal{A}^l[\textsc{AND}_S]$ kept, at the same time. A naïve approach uses the following lemma from \citep{boixadsera2024theorymodeldistillation}. 

\begin{lemma}[Lemma 3.8 \citep{boixadsera2024theorymodeldistillation}]\label{lem:packing}
    Let $\mathcal{S}$ be a collection of non-degenerate $k$-clauses. Let $\mathcal{G} = \{\textsc{AND}_S \mid S \in \mathcal{S}\}$. If $\varphi$ approximately satisfies $\tau$-LRH w.r.t. $\mathcal{G}$: 
    \begin{equation}
        \forall g \in \mathcal{G}, \exists w \in \mathbb{R}^m, \|w\|\leq \tau \text{ and } \E_{x \sim \text{U}[\{0,1\}^d]} [(w \cdot \varphi(x) - g(x))^2] \leq 2^{-k-2}, 
    \end{equation}
    then $|\mathcal{S}| \leq 2^{3k + 4} \tau^2 \E_x \|\varphi(x)\|^2$ 
\end{lemma}

We want to strengthen this to the following:
\begin{lemma}[Packing with for-loops]\label{lem:packing2}
    Let $\mathcal{S}$ be a collection of non-degenerate $k$-clauses. If $\varphi$ approximately satisfies local iteration alignment w.r.t. $\mathcal{S}$: 
    \begin{equation}
        \forall S \in \mathcal{S}, \exists w \in \mathbb{R}^m, \|w\|\leq \tau \text{ and } \E_{x \sim \text{U}[\{0,1\}^d]} [(w \cdot \varphi(x) - \mathcal{A}^l[[\textsc{AND}_S]](x))^2] \leq 2^{-\Theta(kl)}, 
    \end{equation}
    then $|\mathcal{S}| \leq 2^{\Theta(kl)} \tau^2 \E_x \|\varphi(x)\|^2$ .
\end{lemma}

\begin{proof}
To demonstrate the structure of $\mathcal{A}^l[\textsc{AND}_S]$, we will perform a loop unrolling. Recall that the input to the inner tree has three parts: some bits to specify the vertex, which we will denote $x_{v,1} \ldots x_{v,\log n}$; some bits to query the graph adjacency matrix $x_{e,1} \ldots x_{e,\binom{n}{2}}$ and some bits to query the DP table $x_{dp,1} \ldots x_{dp,n}$. Fix a $k$-clause $S$. Let its index set be $\{v_1, \ldots, v_a; e_1, \ldots, e_b; dp_1, \ldots, dp_c\}$ where $a + b + c = k$ and denote by $z_I$ the literal for $x_I$ for some $I$ in the index set.  We have, for some initialization vector and graph adjacency $A$:
\begin{align}
    \mathcal{A}^l[\textsc{AND}_S](\textsc{Init}, A) &:= h_{n,l}(\textsc{Init}, A)\\
    &= \bigwedge_{i \in [a]} z_{v_i}(n)  \wedge\bigwedge_{j \in [b]} z_{e_j} (A)\wedge\bigwedge_{u \in [c]} z_{dp_u,l - 1}(\textsc{Init}, A)
\end{align}

Now, the conjunctions $\bigwedge_{i \in [a]} z_{v_i}$ defines a bipartition of the vertex set into two parts,. Denote by $c_{S,v} \in \{0,1\}$ the indicator function for the set carved out by $\bigwedge_{i \in [a]} z_{v_i}$ and remark that we know $c_{S,v}$ even before seeing any input (and thus they are constants w.r.t. $\mathcal{A}^l[\textsc{AND}_S]$). Furthermore, we write $\bigwedge_{j \in [b]} z_{e_j}(A)$ as $S(A)$ since this quantity depends only on the input graph. Thus:
\begin{align}
    \mathcal{A}^l[\textsc{AND}_S](\textsc{Init}, A) &= c_{S,n}  \wedge S(A)\bigwedge_{u \in [c]} z_{dp_u,l - 1}(\textsc{Init}, A)\\
    &= c_{S,n}  \wedge S(A) \wedge \bigwedge_{u \in [c]^+} c_{S,dp_u}  \wedge S(A) \wedge\bigwedge_{u_2\in [c]} h_{dp_{u_2},l - 2}(\textsc{Init}, A)\nonumber\\
    & \wedge\bigwedge_{u\in [c]^-}  \neg  \left(c_{S,dp_u}  \wedge S(A) \wedge\bigwedge_{u_2\in [c]} h_{dp_{u_2},l - 2}(\textsc{Init}, A)\right)
\end{align}

At layer $i$ of the for-loop, evaluating an entry asks for $c$ evaluation of the previous layer, naïvely giving us $c^l$ evaluations when completely unrolling all $l$ layers of $\mathcal{A}^l$. However, because the evaluations are only at the indices $dp_1,\ldots,dp_c$ which are determined by $S$, we do not need to fill out the whole DP table but only at these $c$ points. Thus, $\mathcal{A}^l[\textsc{AND}_S](\textsc{Init}, A)$ depends on $\textsc{Init}$ only through the $c$ bits and on $A$ only through $S(A)$ which looks at the $b$ bits of $A$. Thus, $\mathcal{A}^l[\textsc{AND}_S]$ is a function of $(b+c)$ bits ($\leq k$ bits) of its input, i.e., a $(b+c)$-junta.

Having established that $\mathcal{A}^l[\textsc{AND}_S]$ are functions of at most $k$ bits, we can use the same packing bound based on Fourier-analytic arguments of \citep{boixadsera2024theorymodeldistillation}. 

Recall that $\tau$-local-iteration alignment implies that: for every $S \in Z'_T$, there is a $w_S \in \mathbb{R}^m$ with $\|w_S\| \leq \tau$ such that, $\langle w, \varphi(x)\rangle = \mathcal{A}^l[\textsc{AND}_S(x)]$ for all $x \in \{-1,1\}^d$ (note that here we use the domain $\{\pm 1\}^d$ that works better with Fourier analytic arguments of boolean functions). Collect all such $w$ into the rows of a matrix $W \in \mathbb{R}^{|\mathcal{S}| \times m}$ and $\varphi(x)$ into the columns of some $\Phi \in \mathbb{R}^{m \times 2^d}$. We can derive the packing bound from the fact that orthogonal projection to the row span of $V \in \mathbb{R}^{\binom{d}{k} \times 2^d}, V_{A,x} = \chi_A(x)$ for all $A \in \binom{[d]}{k}$ and $\chi_A$ being the parity function of indices in $A$, has large norm. Let $P \in \mathbb{R}^{2^d \times 2^d}$ be the orthogonal projection to this subspace of low degree polynomials in $L^2(\{\pm 1\}^d)$ and $P^\top$ the projection to the orthogonal subspace 

First, we want to compute the coefficient of the degree $k$ term of $\mathcal{A}^l[\textsc{AND}_S](\textsc{Init}, A)$ when written as $\{\pm 1\}^d$-polynomial gives:
\begin{align}
    \mathcal{A}^l[\textsc{AND}_S](\textsc{Init}, A) &= h_{n,l}\\
    &= 2^{-(c+2) +1} (1+c_{S,n}) \cdot (1+S(A)) \cdot
    \prod_{u\in [c]}\left(1+ \sigma h_{dp_{u},l - 1}(\textsc{Init}, A))\right)-1
\end{align}
where $\sigma$ is either $1$ or $-1$ depending on the sign of the literal. Thus the leading term is of degree $k$ has coefficient $\pm 2^{-\Theta(kl)}$. Therefore, 
\begin{equation}
    [W\Phi P_k]_{S,x} = \pm 2^{-\Theta(kl)} \chi_{I(S)}, 
\end{equation}
where $I(S)$ is the set of indices of the input that appears in literals of $S$.

Therefore, 
\begin{align}
    W\Phi \Phi^\top W^\top \succeq W\Phi  P P^\top\Phi^\top W^\top = 2^{d - \Theta(kl)} I,
\end{align}
and one conclude that $|\textsc{det}(W\Phi \Phi^\top W^\top)| \geq 2^{(
d - \Theta(kl))|\mathcal{S}|}$. 

Using $\tau$-local-iteration alignment, we get the following for free:
\begin{lemma}[Claim B.4 \citep{boixadsera2024theorymodeldistillation}]
    We have: $\textsc{det}(W\Phi \Phi^\top W^\top) \leq (2^d (\E\|\varphi\|^2)^2 \tau^2/|\mathcal{S}|)^{|\mathcal{S}|}.$
\end{lemma}
\end{proof}
Combining the two gives:
\begin{equation}
    2^{(
d - \Theta(kl))|\mathcal{S}|} \leq  (2^d (\E\|\varphi\|^2)^2 \tau^2/|\mathcal{S}|)^{|\mathcal{S}|},
\end{equation}
which gives $|\mathcal{S}| \leq 2^{\Theta(kl)} (\E\|\varphi\|^2)^2 \tau^2$.

Finally, we give details on the DP procedures that stitch together all the root-prefix paths to find the true tree based on 0-1 loss.

\paragraph{Dynamic programming algorithm to infer final tree}  
For this section, recall that the concept class dictates that the decision tree can be different when processing different vertices of the graph. To this end, we let $t_i$ be the true decision tree for vertex $i$ and treat them as a different decision tree. To get back the full tree, we simply add a $\log n$-depth index selector at the beginning of the estimated tree. 

The problem is reduced to inferring $t_i$'s simultaneously. Recall that the set of possible clauses for tree $i$ is $\mathcal{S}_i$. For some weight function $u: \bigotimes_i \mathcal{S}_i \to \mathbb{R}$, define the valuation function:
\begin{equation}
    \text{val}(\widetilde{T}_1, \ldots, \widetilde{T}_n, u) := \sum_{S_i \in \text{Leaves}(T_i), i \in [n]} (2\widetilde{T}(S_1,\ldots,S_n) -1) u_{S_1,\ldots,S_n},
\end{equation}where $\widetilde{T}(S_1,\ldots,S_n)$ is defined as the computation of the template function $\mathcal{A}[\widetilde{T}]$ by replacing $\widetilde{T}_i$ with $S_i$ in only the first layer. In other words, even though we do not have any input $x$ to evaluate $\mathcal{A}[\widetilde{T}]$ at, we can still use the leaf node of the paths $S_1,\ldots,S_n$ to determine the hidden representation $h_{v,1}$  of the first layer, for each $v$'s. Determining all $h_{v,1}$ allows us to compute $\mathcal{A}[\widetilde{T}]$ by passing it to the next layers.  

Recall that the weight function for our dynamic program is: 
\begin{equation}
    v: \bigotimes_i \mathcal{S}_i \to \mathbb{R}, v_{S_1, \ldots, S_n} := \mathbb{E}_{x \sim \mathcal{D}}\left[(2f_\theta(x) -1)\prod_i\textsc{AND}(S_i) \right].
\end{equation}

Using this weight function, one gets a negative correlation of the $0$-$1$ loss:
\begin{align}
    \text{val}(\widetilde{T}_1, \ldots, \widetilde{T}_n, v) &= \sum_{S_i \in \text{Leaves}(T_i), i \in [n]} (2\widetilde{T}(S_1,\ldots,S_n) -1) \mathbb{E}_{x \sim \mathcal{D}}\left[(2f_\theta(x) -1)\prod_i\textsc{AND}(S_i) \right]\\
    &= \mathbb{E}_{x \sim \mathcal{D}}\left[(2f_\theta(x) - 1)\left(\sum_{S_i \in \text{Leaves}(T_i), i \in [n]}(2\widetilde{T}_i(S_i) -1)\prod_i\textsc{AND}(S_i) \right)\right].
\end{align}

Since for each input $x$ in the domain, by the definition of a decision tree, exactly one path $S_i$ is traversed for each tree at node $i$. For these correct path, $\widetilde{T}(S_1,\ldots,S_n) = \widetilde{T}(x)$. Thus we have:
\begin{align}
    \text{val}(\widetilde{T}_1, \ldots, \widetilde{T}_n, v) &= \mathbb{E}_{x \sim \mathcal{D}}\left[(2f_\theta(x) - 1)(2\widetilde{T}(x) - 1)\right],
\end{align} which is the $0-1$ loss for the estimated trees. 





In our algorithm, we use Hoeffding inequality to approximate $v$ with random sampling. Note that this step requires, in the worst case, with probability at least $1 -\delta/2$, approximating $|\mathcal{S}|^n$ entries of $v$ naively with error at most $\epsilon/\prod_i s_i$ where $s_i$ is a bound on the number of leaves of $t_i$ (naively $|\mathcal{S}|$), and runs in time $\text{poly}(m',n)$ where $m' = \text{poly}(1/\epsilon, \log(|S|^n/\delta))$ is the number of draws to obtain the Hoeffding bound. When choosing $R = r$ in the algorithm, $|\mathcal{S}|$ is of order $2^{O(lr)}(\E \|\varphi\|)^2 \tau^2$ based on the previous bound on $|\mathcal{S}|$ and the approximation of $v$ is done in $\text{poly}(n,l,r,1/\epsilon, \log(1/\delta))$. 

Finally, we can run the dynamic program that computes for each $S_1 \in \mathcal{S}_1$, each tree size $s'_1 = 0.. s_1$, for each $S_2 \in \mathcal{S}_2$, each tree size $s'_2 = 0.. s_2$, etc. the best subtrees $\widetilde{T}_i$ of size $s'_i$ rooted at the end of the clause $S_i$, for all $i \in [n]$. The runtime of this DP is computed as $|\mathcal{S}|^n \cdot s^n \cdot \text{poly}(n,l,r,1/\epsilon, \log(1/\delta)) = \text{poly}(2^{nlr}, (\E \|\varphi\|)^2 \tau^2, s^n, n,l,r,1/\epsilon, \log(1/\delta))$.

\section{Experiments}
We report more detailed diagnostics and statistics for our end-to-end implementation of the algorithm in \Cref{tab:algo2_phase2_diagnostics}. 
\begin{table*}[]
    \centering
    \begin{tabular}{ccccc}
depth & $k$ & paths / vertex & \begin{tabular}{c} candidate trees \\ / vertex \end{tabular} & source agreement\\
\hline
2 & 10 & \begin{tabular}{c}$68/47/88$\\$107/25/25$\end{tabular} & $2/2/2/2/2/2$ & $0.753$\\
2 & 50 & \begin{tabular}{c}$174/159/214$\\$189/201/68$\end{tabular} & $3/3/3/3/3/2$ & $0.748$\\
2 & 100 & \begin{tabular}{c}$200/200/200$\\$200/200/126$\end{tabular} & $2/3/2/2/3/2$ & $0.755$\\
2 & 200 & \begin{tabular}{c}$200/200/200$\\$200/200/210$\end{tabular} & $2/2/2/2/2/2$ & $0.755$\\
\hline
3 & 10 & \begin{tabular}{c}$127/86/147$\\$147/25/25$\end{tabular} & $1/1/1/1/1/1$ & $0.903$\\
3 & 50 & \begin{tabular}{c}$235/223/250$\\$247/252/125$\end{tabular} & $1/1/1/1/1/1$ & $0.906$\\
3 & 100 & \begin{tabular}{c}$302/281/293$\\$242/295/209$\end{tabular} & $1/1/1/1/1/1$ & $0.900$\\
3 & 200 & \begin{tabular}{c}$342/313/336$\\$294/312/319$\end{tabular} & $1/1/1/1/1/1$ & $0.907$\\
\hline
4 & 10 & \begin{tabular}{c}$180/104/219$\\$183/25/25$\end{tabular} & $2/1/1/1/1/1$ & $0.645$\\
4 & 50 & \begin{tabular}{c}$246/258/345$\\$254/324/212$\end{tabular} & $1/1/1/1/1/2$ & $0.649$\\
4 & 100 & \begin{tabular}{c}$432/391/399$\\$330/294/255$\end{tabular} & $1/1/1/2/1/2$ & $0.652$\\
4 & 200 & \begin{tabular}{c}$424/477/377$\\$424/476/347$\end{tabular} & $1/1/1/2/1/2$ & $0.701$\\
\hline
5 & 10 & \begin{tabular}{c}$212/120/277$\\$209/25/25$\end{tabular} & $2/2/2/2/2/2$ & $0.736$\\
5 & 50 & \begin{tabular}{c}$276/303/330$\\$278/368/217$\end{tabular} & $2/3/3/3/3/2$ & $0.740$\\
5 & 100 & \begin{tabular}{c}$444/434/413$\\$319/305/275$\end{tabular} & $3/3/6/2/3/2$ & $0.748$\\
5 & 200 & \begin{tabular}{c}$535/526/475$\\$416/490/379$\end{tabular} & $4/4/5/3/3/2$ & $0.739$\\
\end{tabular}
    \caption{Diagnostic statistics for Phase 2. ``Paths / vertex'' is the number of clauses retained for each per-vertex subtree after the validation-error pruning step, and ``candidate trees / vertex'' is the number of candidate trees generated by the phase-2 DP before the final product search. Source agreement is the agreement between the reconstructed local-iteration algorithm and the trained source network.}
    \label{tab:algo2_phase2_diagnostics}
\end{table*}

\newpage
\subsection{Details on Experimental Setups}

All experiments were implemented in Python (v3.12.13). Neural-network models were built using PyTorch (v2.10.0+cu128), graph construction and synthetic benchmark generation handled through NetworkX (v3.6.1). Numerical computation and data processing used NumPy (v2.2.6), SciPy (v1.15.2). We also use tqdm (v4.67.3) and graphviz (v0.21). Models were trained using the optimization and initialization routines provided by the default PyTorch stack, with Adam used as the optimizer in our learned pipelines.

We ran the experiments on Google Colab, and the hardware configuration used is summarized in Table~\ref{tab:hardware-specs}.

\begin{table}[H]
\centering
\caption{Hardware specifications.}
\label{tab:hardware-specs}
\begin{tabular}{ll}
\toprule
Component & Specification \\
\midrule
Architecture & x86\_64 \\
OS & Ubuntu 22.04.5 LTS \\
CPU & Intel(R) Xeon(R) CPU @ 2.20GHz \\
GPU & NVIDIA A100-SXM4-40GB \\
GPU memory & 40960 MiB (approximately 141 GB) \\
RAM & 83.47 GiB \\
\bottomrule
\end{tabular}
\end{table}

Table~\ref{tab:software-licenses} lists the licenses of the main software libraries used in the experiments.

\begin{table}[H]
\centering
\caption{Main software licenses.}
\label{tab:software-licenses}
\begin{tabular}{ll}
\toprule
Software & License \\
\midrule
PyTorch & BSD-3-Clause \\
NetworkX & BSD-3-Clause \\
NumPy & BSD-3-Clause \\
SciPy & BSD-3-Clause \\
tqdm & MPL-2.0 and MIT\\
graphviz & Eclipse Public License - v 2.0\\
\bottomrule
\end{tabular}
\end{table}

\section{LLM Usage Disclosure}

We used an LLM to help with code writing and with polishing the paper text.

\end{document}